\documentclass[a4paper]{article}
\usepackage[utf8]{inputenc}
\usepackage[T1]{fontenc}
\usepackage{xspace}
\usepackage[a4paper,top=3cm,bottom=2cm,left=2.5cm,right=2.5cm,marginparwidth=1.75cm]{geometry}

\usepackage{amsmath, amsthm, amsfonts, amssymb, bm, bbm, mathtools, nccmath, xfrac}

\usepackage[dvipsnames]{xcolor}
\usepackage{hyperref}
\usepackage{graphicx}
\usepackage{tikz}

\usepackage{natbib}
\usepackage{todonotes}
\usepackage{enumitem}
\date{March 22, 2019}

\newcommand{\MR}{D^\text{MR}}

\newcommand{\SR}{\Pi^\text{SR}}
\newcommand{\SD}{\Pi^\text{SD}}

\newcommand{\ucrl}{{\small\textsc{UCRL2}}\@\xspace}

\newcommand{\ucrlb}{{\small\textsc{UCRL2B}}\@\xspace}

\newcommand{\rmaxbound}{r_{\max}}
\newcommand{\M}{\mathcal M}

\newcommand{\A}{\mathcal A}

\newcommand{\bigO}{\mathcal O}

\renewcommand{\L}{\mathcal L}
\newcommand{\calS}{\mathcal S}

\newcommand{\T}{^\intercal}
\renewcommand{\Re}{\mathbb R}
\newcommand{\Na}{\mathbb N}
\newcommand{\Fil}{\mathbb F}
\newcommand{\F}{\mathcal F}

\newcommand{\Np}[1][s,a]{N_k^+(#1)}

\newcommand{\E}{\mathbb{E}}
\newcommand{\Ex}[1]{\mathbb E \left[ #1 \right]}

\newcommand{\Proba}[1]{\mathbb P \left( #1 \right)}
\newcommand{\nextstates}{\Gamma}

\newcommand{\diameter}{(\rmaxbound D)}
\newcommand{\V}{\mathbb{V}}
\newcommand{\Var}[2]{\mathbb{V}_{#2} \left( #1 \right)}
\newcommand{\Varcond}[2]{\mathbb{V}\left({#1} \big| {#2}\right)}
\newcommand{\one}[1]{\mathbbm{1}\left\{ #1 \right\}}

\newcommand{\sigalg}{$\sigma$-algebra\@\xspace}

\newcommand{\pare}[1]{\left( #1 \right)}
\newcommand{\abs}[1]{\left| #1 \right|}
\newcommand{\rt}[1]{{r}_k(#1)}
\newcommand{\pt}[2]{{p}_k(#2|#1)}
\newcommand{\ph}[2]{\widehat{p}_k(#2|#1)}
\newcommand{\p}[2]{\wb{\wb{p_k}}(#2|#1)}
\newcommand{\h}{h_k}

\newcommand{\ie}{i.e.,\@\xspace}
\newcommand{\eg}{e.g.,\@\xspace}

\newcommand{\st}{s.t.\@\xspace}

\newcommand{\rv}{r.v.\@\xspace}
\newcommand{\as}{a.s.\@\xspace}
\newcommand{\iid}{i.i.d.\@\xspace}

\newcommand{\wt}[1]{\widetilde{#1}}
\newcommand{\wb}[1]{\overline{#1}}

\newcommand{\wh}[1]{\widehat{#1}}

\DeclareMathAccent{\wtilde}{\mathord}{largesymbols}{"65}

\newcommand{\lnn}[1]{\ln\left({#1}\right)}

\DeclareMathOperator*{\argmax}{\arg\,\max}

\newcommand{\SP}[1]{sp\left(#1\right)}

\newtheorem{theorem}{Theorem}

\newtheorem{lemma}[theorem]{Lemma}
\newtheorem{proposition}[theorem]{Proposition}

\newlength{\minipagewidth}
\newlength{\minipagewidthx}
\newcommand{\bookboxx}[1]{\small
\par\medskip\noindent
\framebox[0.99\textwidth]{
\begin{minipage}{0.97\dimexpr\textwidth-\parindent\relax} {#1} \end{minipage} } \par\medskip }

\title{Improved Analysis of UCRL2 with Empirical Bernstein Inequality}
\author{Ronan Fruit, Matteo Pirotta and Alessandro Lazaric}

\begin{document}

\maketitle
\begin{abstract}
        We consider the problem of exploration-exploitation in communicating Markov Decision Processes.
        We provide an analysis of UCRL2 with Empirical Bernstein inequalities (UCRL2B).
        For any MDP with $S$ states, $A$ actions, $\Gamma \leq S$ next states and diameter $D$, the regret of UCRL2B is bounded as $\widetilde{O}(\sqrt{D\Gamma S A T})$.
\end{abstract}

\section{Introduction}
\citet{Jaksch10} introduced the reinforcement learning algorithm \ucrl and proved a regret bound of order $\wt{\mathcal{O}}(DS \sqrt{AT})$ for any communicating MDP with $S$ states, $A$ actions and diameter $D$.
\ucrl used Hoeffding inequalities to build an uncertainty set around rewards and transitions.
\citep{fruit2018constrained} exploited empirical Bernstein inequalities to prove a regret bound of $\wt{\mathcal{O}}(D \sqrt{\Gamma S AT})$ where $\Gamma := \max_{s,a} \Gamma(s,a) \leq S$ is the maximum number of possible next states.
In this document, we show that we can improve the analysis of \ucrl with empirical Bernstein bound (\ucrlb) and we show a regret bound of $\wt{O}(\sqrt{D \Gamma S A T})$.
This document is intended as a support to our tutorial at the 30th International Conference on \href{http://alt2019.algorithmiclearningtheory.org/}{Algorithmic Learning Theory (ALT 2019)}. For a more detailed analysis, please refer to~\citep{fruit2019thesis}.

\section{Preliminaries}
We consider a \emph{communicating} MDP \citep[Sec. 8.3]{puterman1994markov} $M = ( \calS, \A, p, r )$ with state space $\calS$ and action space $\A$.
Every state-action pair $(s,a)$ is characterized by a reward distribution with mean $r(s,a)$ and support in $[0, \rmaxbound]$, and a transition distribution $p(\cdot|s,a)$ over next states.
We denote by $S = |\calS|$ and $A = |\A|$ the number of states and action, by 
$\nextstates(s,a) = \|p(\cdot|s,a)\|_0$ the number of states reachable by selecting action $a$ in state $s$, and by $\nextstates = \max_{s,a} \nextstates(s,a)$ its maximum.
A stationary Markov randomized policy $\pi : \calS \rightarrow P(\A)$ maps states to distributions over actions.
The set of stationary randomized (resp. deterministic) policies is denoted by $\SR$ (resp. $\SD$).
Any policy $\pi \in \SR$ has an associated \emph{long-term average reward} (or gain) and a \emph{bias function} defined as
\begin{align*}
        g^\pi(s) := \lim_{T\to +\infty} \mathbb{E}^\pi_{s} \bigg[ \frac{1}{T}\sum_{t=1}^T r(s_t,a_t)\bigg]~~\text{ and }~~
        h^\pi(s) := \underset{T\to +\infty}{C\text{-}\lim}~\mathbb{E}^\pi_s \bigg[\sum_{t=1}^{T} \big(r(s_t,a_t) - g^\pi(s_t)\big)\bigg],
\end{align*}
where $\mathbb{E}^\pi_{s}$ denotes the expectation over trajectories generated starting from $s_1 = s$ with $a_t \sim \pi(s_t)$.
The bias $h^\pi(s)$
 measures the expected total difference between the reward and the stationary reward in \emph{Cesaro-limit} 
 (denoted by $C\text{-}\lim$). Accordingly, the difference of bias $h^\pi(s)-h^\pi(s')$ quantifies the (dis-)advantage of starting in state $s$ rather than $s'$. We denote by $\SP{h^\pi} := \max_s h^\pi(s) - \min_{s} h^\pi(s)$ the \emph{span} of the bias function.
In weakly communicating MDPs, any optimal policy $\pi^\star \in \argmax_\pi g^\pi(s)$ has \emph{constant} gain, i.e., $g^{\pi^\star}(s) = g^\star$ for all $s\in\calS$. 
Moreover, there exists a policy $\pi^\star \in \argmax_\pi g^\pi(s)$ for which $(g^\star,h^\star) = (g^{\pi^\star},h^{\pi^\star})$ satisfy the \emph{optimality equation},
\begin{equation}\label{eq:optimality.equation}
        \forall s \in \calS, \qquad
        h^\star(s) + g^\star = L h^\star(s) := \max_{a \in \A} \{r(s,a) + p(\cdot|s,a)^\top h^\star\},
\end{equation}
where $L$ is the \emph{optimal} Bellman operator.
Finally,
    $D = \max_{s\neq s'} \{\tau(s \to s')\}$
denotes the diameter of $M$,
where $\tau(s\to s')$ is the minimal expected number of steps needed to reach $s'$ from $s$.

\textbf{Learning Problem.} 
Let $M^\star$ be the true MDP. We consider the learning problem where $\mathcal{S}$, $\mathcal{A}$ and $\rmaxbound$ are \emph{known}, while rewards $r$ and dynamics $p$ are \emph{unknown} and need to be estimated \emph{on-line}. We evaluate the performance of a learning algorithm $\mathfrak{A}$ after $T$ time steps by its cumulative \emph{regret} $\Delta(\mathfrak{A},T) = \sum_{t=1}^T (g^\star-r_t(s_t,a_t))$.

\begin{figure}[t]
\renewcommand\figurename{\small Figure}
\begin{minipage}{\columnwidth}
\bookboxx{
\textbf{Input:} Confidence $\delta \in ]0,1[$, $r_{\max}$, $\calS$, $\A$

\noindent \textbf{Initialization:} Set $t := 1$ and observe $s_1$ and for any $(s,a,s') \in \calS \times \A \times \calS$:
$N_1(s,a) = 0$,
$\wh{p}_1(s'|s,a) = 0$, $\wh{r}_1(s,a) = 0$, 
$\wh{\sigma}_{p,1}^2(s'|s,a) = 0$, $\wh{\sigma}_{r,1}^2(s,a) = 0$

\noindent \textbf{For} episodes $k=1, 2, ...$ \textbf{do}

\begin{enumerate}[leftmargin=4mm,itemsep=0mm,topsep=0mm]
\item Set $t_k \leftarrow t$ and episode counters $\nu_k (s,a) \leftarrow 0$

\item Compute the upper-confidence bounds (Eq.~\ref{eq:confidence_interval_p} and~\ref{eq:confidence_interval_r}) and
        the extended MDP
        $\mathcal{M}_k$ as in Eq.~\ref{eq:extendedmdp}

\item
Compute an $\rmaxbound/t_k$-approximation $\pi_k$ of
 Eq.~\ref{eq:optim_ucrl2}: $(g_k,h_k, \pi_k) = EVI( \L_{\alpha}^k,  \mathcal{G}_{\alpha}^k,\frac{\rmaxbound}{t_k},0 , s_1)$

%
\item Sample action $a_t \sim \pi_k(\cdot|s_t)$

\item \textbf{While} 
        True
        \textbf{do}
\begin{enumerate}
        \item Execute $a_t$, obtain reward $r_{t}$, and observe $s_{t+1}$
        \item Set $\nu_k (s_t,a_t) \leftarrow \nu_k (s_t,a_t) + 1$ 
        \item \textbf{If} $\nu_k(s_t,a_t) \geq \max\{1, N_k(s_t,a_t)\}$ \textbf{then}
                \begin{itemize}
                        \item Set $t \leftarrow t + 1$ and {\color{red!50!black}break}
                \end{itemize}
        \item \textbf{Else}
                \begin{itemize}
                        \item Sample action $a_{t+1} \sim \pi_k(\cdot|s_{t+1})$ and set  $t \leftarrow t + 1$
                \end{itemize}
\end{enumerate}

\item Set $N_{k+1}(s,a) \leftarrow N_{k}(s,a)+ \nu_k(s,a)$ 
\item Update statistics (\ie $\wh{p}_{k+1}, \wh{r}_{k+1}, \wh{\sigma}_{p,k+1}^2$ and $\wh{\sigma}_{r,k+1}^2$) \tikz[baseline]{\node[anchor=base](alg_lastline){};}
\end{enumerate}
}
\caption{\ucrlb algorithm.}
\label{fig:ucrlb}
\end{minipage}
\end{figure}

\section{UCRL2B}

\ucrlb is a variant of \ucrl~\citep{Jaksch10} that construct confidence intervals based on the empirical Bernstein inequality~\citep{audibert2007tuning} rather than Hoeffding's inequality.
As \ucrl, \ucrlb proceeds through episodes $k = 1, 2 \ldots$. At the beginning of each episode $k$, UCRL computes a set of plausible
MDPs defined as 
\begin{align}\label{eq:extendedmdp}
        \mathcal{M}_k = \bigg\{ M = \langle \calS, \A, \wt r, \wt p \rangle \; :\; \wt r(s,a)  \in B_r^k(s,a), \wt p(s'|s,a) \in B_p^k(s,a,s'), \sum_{s'} \wt p(s'|s,a) = 1 \bigg\},
\end{align}
where $B_r^k$ and $B_p^k$ are high-probability confidence intervals on the rewards and transition probabilities of the true MDP $M^\star$, which guarantees that (see App.~\ref{app:conf_interval})
\[\Proba{ \exists k \geq 1, \text{ \st~} M^\star \not\in {\mathcal{M}}_k } \leq \frac{\delta}{3}.\]
As mentioned,  we use confidence intervals constructed using empirical Bernstein’s inequality~\citep[][Thm. 1]{Audibert:2009:ETU:1519541.1519712}
\begin{align}
          \beta_{p,k}^{sas'} &:= 2 \sqrt{\frac{ \wh{\sigma}^2_{p,k}(s'|s,a)}{\Np} \lnn{\frac{6 SA \Np}{\delta}}} + \frac{6 \lnn{\frac{6 SA \Np}{\delta}}}{\Np} \label{eq:bernstein_confidence_bound_p}\\
     \beta_{r,k}^{sa} &:= 2 \sqrt{\frac{ \wh{\sigma}^2_{r,k}(s,a)}{\Np} \lnn{\frac{6 SA \Np}{\delta}}} + \frac{6 \rmaxbound \lnn{\frac{6 SA \Np}{\delta}}}{\Np} \label{eq:bernstein_confidence_bound_r}
\end{align}
where $N_k(s, a)$ is the number of visits in $(s, a)$ before episode $k$, $N_k^+(s,a) = \max\{ 1, N_k(s,a)\}$, $\wh{\sigma}^{2}_{p,k}$ and $\wh{\sigma}^2_{r,k}$ are the population variance of transition and reward function at episode $k$.
We define by $\wh r_k$ and $\wh p_k$ the empirical average of rewards and transitions:
\begin{align*}
        \wh r_k(s,a) := \frac{1}{N_k(s,a)} \sum_{t=1}^{t_k-1} \one{s_t,a_t=s,a} \cdot r_t ~~\text{ and }~~
        \wh p_k(s'|s,a) := \frac{1}{N_k(s,a)} \sum_{t=1}^{t_k-1} \one{s_t,a_t,s_{t+1}=s,a,s'}
\end{align*}
where $t_k$ is the starting time of episode $k$
The estimated transition probability $\wh{p}_k(s'|s,a)$ correspond to the sample mean of i.i.d. Bernouilli \rv with mean ${p}(s'|s,a)$ and therefore the population variance can be easily computed as $\wh{\sigma}^2_{p,k}(s'|s,a) := \wh{p}_k(s'|s,a)\left( 1 - \wh{p}_k(s'|s,a) \right)$.
The population variance of the reward can be computed recursively at the end of every episode:
\begin{align*}
 \wh{\sigma}^2_{r,k+1}(s,a) &:= \frac{1}{N_{k+1}^+(s,a)} \left(\sum_{l=1}^{k} S_l(s,a) \right) - \left(\wh{r}_{k+1}(s,a)\right)^2\\
 &~= \frac{S_k(s,a)}{N_{k+1}^+(s,a)}  + \frac{N_{k}(s,a)}{N_{k+1}^+(s,a)}\left(\wh{\sigma}^2_{r,k}(s,a) + \left(\wh{r}_{k}(s,a)\right)^2 \right) - \left(\wh{r}_{k+1}(s,a)\right)^2.
\end{align*}
where $S_k(s,a) := \sum_{t=1}^{t_k -1} \one{s_t,a_t=s,a} \cdot r_t^2$.
The extended MDP $\mathcal{M}_k$ is defined by the compact sets
\begin{align}
   B_p^k(s,a,s') &:= \left[\wh{p}_k(s'|s,a) - \beta_{p,k}^{sas'},\wh{p}_k(s'|s,a) + \beta_{p,k}^{sas'}\right] \cap \big[0,1\big] \label{eq:confidence_interval_p} \\
   B_r^k(s,a) &:= \left[\wh{r}_k(s,a) - \beta_{r,k}^{sa}, \wh{r}_k(s,a) + \beta_{r,k}^{sa}\right] \cap \big[0, \rmaxbound\big] \label{eq:confidence_interval_r}
\end{align}

As \ucrl, \ucrlb executes a policy $\pi_k$ which is an approximate solution to the following optimization problem:
\begin{align}\label{eq:optim_ucrl2}
        g_k^\star := \sup_{M' \in \M_k} \left\{\max_{\pi \in \Pi^{\text{SD}}} g^\pi_{M'}  \right\} = \sup_{M' \in \M_k} g^\star_{M'}  .
\end{align}
Since $M^\star \in \mathcal{M}_k$ w.h.p., it holds that $g_k^\star \geq g^\star_{M^\star}$.
An approximated solution can be computed using Extended Value Iteration (EVI)~\citep{Jaksch10}. 
For technical reasons, we do not apply EVI directly to $\mathcal{M}_k$ but to $\mathcal{M}^k_\alpha$, where $\alpha$ is the coefficient of the aperiodicity transformation.
EVI iteratively applies the following extended aperiodic optimal Bellman operator $\L_{\alpha}^k$:
\begin{align}\label{eq:extended_bellman_operator4}
 \L_{\alpha}^k v (s) :=\underset{a\in \A_s}{\max} \left\{ \underset{r \in B_r(s,a)}{\max}\left\{ r \right\} + \alpha \cdot \max_{p\in B_p^k(s,a)} \left\{ p\T v \right\} \right\} + (1-\alpha)\cdot v(s).
\end{align}
where $B_p^k(s,a):= \left\{ p \in \Delta_S:~ p(s') \in B_p^k(s,a,s'),~ \forall s' \in \calS \right\}$ and $\Delta_S$ is the $S$-dimensional simplex.
We arbitrarily set $\alpha = 0.9$.
We recall that, by properties of the aperiodicity transformation, the optimal gains of $\M_{\alpha}^k$ and $\M_k$ are equal (denoted by $g_k^\star$).
If we ran EVI (see Alg.~\ref{alg:vi}) on $\mathcal{M}_{\alpha}^k$ with accuracy $\epsilon_k = \rmaxbound / t_k$, we have that
\begin{align}
 &|g_k - g_k^\star| \leq \varepsilon_k/2~ := \frac{\rmaxbound}{2t_k} \\
 \text{and }~  &\| \L_{\alpha}^k h_k - h_k -g_k e \|_{\infty} \leq \varepsilon_k := \frac{\rmaxbound}{t_k}. \label{eq:near.optimality.equation}
\end{align}
where $(g_k,h_k, \pi_k) = EVI( \L_{\alpha}^k,  \mathcal{G}_{\alpha}^k,\frac{\rmaxbound}{t_k},0 , s_1)$.\footnote{The extended greedy operator is defined as 
\begin{align}\label{eq:greedy.operator.ucrl2}
        \forall s\in\calS, \forall v\in\Re^S, ~~\mathcal{G}_k v(s) \in \argmax_{a\in\A_s} \left\{ \max_{r \in B_r^k(s,a)} r  + \max_{p\in B_p^k(s,a)} p\T v \right\}.
\end{align}
}
We denote by $r_k$ and $p_k$ the optimistic reward and transitions at episode $k$.


\paragraph{Regret Bound.}
We can now provide the improved regret bound for \ucrlb\\[.3cm]
\tikz[baseline]{
        \node[text width=0.98\textwidth, fill=CornflowerBlue!5, inner sep=2pt] {
\begin{theorem}\label{thm:regret.bound2}
There exists a numerical constant $\beta >0$ such that for \textup{any} communicating MDP, with probability at least $1-\delta$, it holds that for all initial state distributions $\mu_1 \in \Delta_S$ and for all time horizons $T>1$
  \begin{equation}\label{eq:regret.bound2}
  \begin{aligned}
          \Delta(\text{\ucrlb},T) \leq \beta &\cdot  \rmaxbound\sqrt{ D \left(\sum_{s,a} \nextstates(s,a)\right) T \lnn{\frac{T}{\delta}}\lnn{T}} \\
  &+ \beta \cdot \rmaxbound D^2  S^2 A  \lnn{\frac{T}{\delta}}\lnn{T}
  \end{aligned}
 \end{equation}
\end{theorem}
\vspace{20pt}
};
}

\citet{Jaksch10} showed that up to a multiplicative numerical constant, the regret of \ucrl is bounded by $\rmaxbound DS\sqrt{AT\lnn{T/\delta}}$. After noticing that $\sum_{s,a}\nextstates(s,a) \leq \nextstates SA$ we can simplify the bound in \eqref{eq:regret.bound2} as \[ \beta \cdot \rmaxbound \sqrt{D \nextstates SA T \lnn{T/\delta}} + \beta \cdot \rmaxbound D^2 S^2 A \lnn{T/\delta} \lnn{T}\]

\begin{figure}[t]
\renewcommand\figurename{\small Figure}
\begin{minipage}{\columnwidth}
\bookboxx{
\textbf{Input:} Bellman operator $L: \Re^S \mapsto \Re^S$, greedy policy operator $G: \Re^S \mapsto \MR$, accuracy $\varepsilon \in ]0,\rmaxbound[$, initial vector $v_0\in\Re^S$, arbitrary reference state $\wb{s} \in \calS$ 

\noindent \textbf{Initialization:} $n=0$, $v_1 = L v_0$  

\noindent \textbf{While} $\SP{v_{n+1} - v_n} > \varepsilon$ \textbf{do}

\begin{enumerate}[leftmargin=4mm,itemsep=0mm,topsep=0mm]
\item Increment $n \leftarrow n+1$
\item Shift $v_n \leftarrow v_n - v_n(\wb{s})e$
\item $(v_{n+1}, d_n) := (L v_{n}, G v_{n})$
\end{enumerate}

\noindent Set $g := \frac{1}{2} \Big( \max\{v_{n+1} - v_n\} + \min\{v_{n+1} - v_n\} \Big)$, $h := v_n$ and $\pi := (d_n)^\infty$

\noindent \textbf{Return} gain $g$, bias $h$, policy $\pi$

}
\caption{(Relative) Value Iteration.}
\label{alg:vi}
\end{minipage}
\end{figure}


\section{Improved regret analysis for \ucrlb}\label{chap:ucrlb:sec:regret.proof.improved}
We now report the standard regret decomposition~\citep[\eg][]{fruit2018constrained}.
The regret after $T$ time steps is defined as
$
 \Delta(\text{\ucrlb},T) = \sum_{t=1}^T \Big(g^\star - r_t\Big)
$. To begin with, we replace $r_t$ by its expected value \emph{conditioned} on the current state $s_t$ using a martingale argument.
Let's denote by $\nu_k(s):= \sum_{a \in \A_ s} \nu_k(s,a)$ the total number of visits in state $s$ during episode $k$.
Defining ${\color{CadetBlue}\Delta_k}:=\sum_{s \in \mathcal{S}} \nu_k(s) \left(g^\star - \sum_{a \in \mathcal{A}_{s_t}} \pi_{k}(a|s)r(s,a)\right)$ the pseudo-regret of episode $k$, it holds with probability at least $1-\frac{\delta}{6}$ that for all $T\geq 1$:
\begin{align}\label{eqn:splitting}
 R(\text{\ucrlb},T) &\leq \sum_{t=1}^T \bigg(g^* - \sum_{a \in\mathcal{A}_{s_t}} \pi_{k_t}(s_t,a) r(s_t,a) \bigg) + 2\rmaxbound\sqrt{T\ln \left(\frac{4T}{\delta}\right)}\nonumber\\
 &= \sum_{k=1}^{k_T} \sum_{s \in \calS} \nu_k(s)  \bigg( g^* - \sum_{a \in\mathcal{A}_{s}} \pi_{k}(a|s) r(s,a) \bigg) + 2\rmaxbound\sqrt{T\ln \left(\frac{4T}{\delta}\right)}\nonumber\\
 &= \sum_{k=1}^{k_T} {\color{CadetBlue}\Delta_k} + 2\rmaxbound\sqrt{T\ln \left(\frac{4T}{\delta}\right)}
\end{align}
where $k_T = \sup \{ k\geq 1 : t \geq t_k\}$. By using optimism and the Bellman equation, we further decompose ${\color{CadetBlue} \Delta_k }$ as~\citep[see \eg][for more details]{fruit2018constrained,fruit2019thesis}
\begin{align*}
  {\color{CadetBlue}\Delta_k} \leq {\color{Orange}\Delta_k^p} +  {\color{PineGreen}\Delta_k^r} + \frac{3\varepsilon_k}{2} \sum_{s \in \calS} \nu_k(s)
\end{align*}
with ${\color{PineGreen}\Delta_k^r} = \sum_{s \in \mathcal{S}} \sum_{a \in\mathcal{A}_{s}} \nu_k(s)\pi_{k}(a|s) \Big({r}_k(s,a) - r(s,a) \Big)$ and 
\begin{align}\label{eq:decomposition.1}
\begin{split}
{\color{Orange}\Delta_k^p} = \underbrace{\alpha \sum_{s,a,s'} \nu_k(s) \pi_{k}(a|s) \Big( {p}_k(s'|s,a) - {p}(s'|s,a) \Big) h_k(s') }_{:={\color{Mahogany}\Delta_k^{p1}}} \\
 + \underbrace{\alpha \sum_{s} \nu_k(s) \left(\sum_{a,s'}\pi_{k}(a|s) {p}(s'|s,a) h_k(s') - h_k(s) \right)}_{:={\color{Fuchsia}\Delta_k^{p2}}}
 \end{split}
\end{align}
where $\alpha \in ]0,1]$ is the coefficient of the \emph{aperiodicity transformation} applied to extended MDP $\M_k$ (in most cases, this coefficient can be taken equal to 1 but we include it for the sake of generality) and $p_k$ is the optimistic kernel at episode $k$. We also consider the general case where the optimistic policy $\pi_k$ can be \emph{stochastic} (in most cases this is not necessary).

We define the event $E^C = \big\{ \exists T > 0, \exists k > 0, ~s.t.~ M^\star \notin \mathcal{M}_k \big\}$. We recall that the probability of this event is small, see App.~\ref{app:conf_interval}:
\[
        \mathbb{P}(E^C) \leq \frac{\delta}{3}
\]

Finally, with probability at least $1-\frac{\delta}{6}$ (and assuming event $E$ holds)~\citep[see \eg][]{fruit2019thesis}:
\begin{align}\label{eqn:bound_reward}
 \forall T\geq 1, ~~\sum_{k=1}^{k_T} {\color{PineGreen}\Delta_k^r} \leq 2 \sum_{k=1}^{k_T} \sum_{s,a} \nu_k(s,a) \beta_{r,k}^{sa} + 4\rmaxbound\sqrt{T\ln \left(\frac{4T}{\delta}\right)}
\end{align}

\subsection{From $D$ to $\sqrt{D}$: Variance Reduction Method}

We will now prove Thm.~\ref{thm:regret.bound2}.
In order to improve the dependency of the regret bound in $D$ (\ie replace $D$ by $\sqrt{D}$), we refine our analysis with three key improvements:
\begin{enumerate}
 \item We leverage on \emph{Freedman's inequality} \citep{freedman1975} instead of Azuma's inequality to bound the MDS. We recall this inequality in Prop.~\ref{prop:fi} below.
 \item We use a \emph{tighter bound} than H\"older's inequality to upper-bound the sum $\sum_{k=1}^{k_T}{\color{NavyBlue}\Delta_k^{p3}}$.
 \item We shift the optimistic bias $h_{k_t}$ by a different constant \emph{at every time step} $t\geq 1$ rather than only at every episode $k\geq 1$. More precisely, the optimistic bias is shifted by a different constant for every episode $k\geq 1$ and for every visited state $s \in \calS$.
\end{enumerate}
To the best of our knowledge, Thm.~\ref{thm:regret.bound2} and its proof are new although it is largely inspired by what is often referred to as \emph{``variance reduction methods''} in the literature \citep{Munos99influenceand,Lattimore12pacbounds, AzarMK13, Lattimorenearoptimalpac,pmlr-v70-azar17a}. Similar techniques are used by \citep{pmlr-v70-azar17a} to achieve a similar bound but in the \emph{finite horizon setting}. This approach is also related to \citep{TalebiKLUCRL} and \citep{MaiManMan14} (in the latter, the variance is called the distribution-norm instead of the variance).

\begin{proposition}[Freedman's inequality]\label{prop:fi}
 Let $(X_n, \F_n)_{n\in \Na}$ be an MDS such that $|X_n| \leq a$ \as for all $n \in \Na$. Then for all $\delta \in ]0,1[$,
 \begin{align*}
 \Proba{ \forall n \geq 1,~ \left| \sum_{i =1}^{n} X_i \right|  \leq 2 \sqrt{ \left( \sum_{i =1}^{n} \Varcond{X_i}{\F_{i-1}} \right) \cdot\lnn{ \frac{4n}{\delta} }  } + 4 a \lnn{\frac{4n}{\delta}} } \geq 1 - \delta
 \end{align*}
\end{proposition}

For any \emph{vector} $u \in \Re^S$, we slightly abuse notation and write $u^2 := u \circ u$ the \emph{Hadamard product} of $u$ with itself. For any probability distribution $p$ over states $\calS$ and any vector $u \in \Re^S$ we define \[\Var{u}{p} := p\T u^2 - (p\T u)^2 = \E_{X \sim p}[u(X)^2] - \big(\E_{X \sim p}[u(X)]\big)^2\] the \emph{``variance''} of $u$ with respect to $p$.
For the sake of clarity we introduce new notations for the transition probabilities: $p_{k}(s'|s):= \sum_{a \in \A_{s} }\pi_{k}(s,a)p_{k}(s'|s,a)$, $\p{s}{s'}:= \sum_{a \in \A_{s} }\pi_{k}(s,a)p(s'|s,a)$ and $\wh{p}_k(s'|s):= \sum_{a \in \A_{s} }\pi_{k}(s,a)\wh{p}_k(s'|s,a)$, for every $s,s' \in \calS$ and every $k\geq 1$.

We start with a new bound relating ${\color{Mahogany}\Delta_k^{p1}}$. We define ${\color{NavyBlue}\Delta_k^{p3}} := \alpha \sum_{s,a,s'} \nu_k(s,a)\left( {p}_k(s'|s,a) - {p}(s'|s,a) \right) h_k(s')$.
\begin{lemma}\label{lem:mds_delta_p1_2}
 Under event $E$, with probability at least $1-\frac{\delta}{6}$:
 \begin{align}\label{eq:mds_delta_p1_2}
  \forall T \geq 1, ~ \sum_{k=1}^{k_T}{\color{Mahogany}\Delta_k^{p1}} &\leq \sum_{k=1}^{k_T}{\color{NavyBlue}\Delta_k^{p3}} + 4 \diameter \lnn{\frac{24T}{\delta}} \nonumber\\ + ~2 &\sqrt{S \lnn{\frac{24T}{\delta}}} \left( \sqrt{\sum_{t=1}^{T}  \Var{\alpha h_{k_t}}{p_{k_t}(\cdot|s_t)} } + \sqrt{\sum_{t=1}^{T}  \Var{\alpha h_{k_t}}{\wb{p}_{k_t}(\cdot|s_t)} } \right)
 \end{align}
\end{lemma}
\begin{proof}
        We use a martingale argument and Prop.~\ref{prop:fi}~\citep[see][]{fruit2019thesis}.
\end{proof}
We \emph{refine} the upper-bound of ${\color{NavyBlue}\Delta_k^{p3}}$ derived by~\citet{Jaksch10}. Instead of bounding the scalar product $(p_k(\cdot|s,a) -p(\cdot|s,a))\T w_k$ by  $\|p_k(\cdot|s,a) -p(\cdot|s,a)\|_1\T \|w_k\|_{\infty}$ using H\"older's inequality, we bound it by $ \sum_{s'}|p_k(s'|s,a) -p(s'|s,a)|\cdot |w_k(s')|$ using the triangle inequality.
Since $\sum_{a,s'}p_k(s'|s,a)  = \sum_{a,s'}p(s'|s,a) = 1$ we can shift $h_k$ by an arbitrary scalar $\lambda_k^{s} \in \Re$ for all $k\geq 1$ and all $s\in \calS $, \ie $w_k^{s} := h_{k} +\lambda_k^{s} e$. Unlike in \ucrl, we choose a \emph{state-dependent} shift, namely $\lambda_k^{s} := - \sum_{a,s'}\wh{p}_{k}(s'|s,a)\pi_k(s,a) h_{k}(s') = - \wh{p}_{k}(\cdot|s)\T h_k$.
It is easy to see that $\SP{w_k^{s}} = \SP{h_k}$ and $\|w_k^{s}\|_\infty \leq \SP{h_k}$ implying that under event $E$, $\|w_k^{s}\|_\infty \leq \diameter/\alpha$.

Using the triangle inequality and the fact that $p_k(s,a) \in B_p^k(s,a)$ by construction and $p(s,a) \in B_p^k(s,a)$ under event $E$: \[\big|p_k(s'|s,a) -p(s'|s,a)\big| \leq \big|p_k(s'|s,a) -\wh{p}_k(s'|s,a)\big| +\big|\wh{p}_k(s'|s,a) -p(s'|s,a)\big| \leq 2 \beta_{p,k}^{sas'}\]
As a result we can write:
\begin{align*}
 {\color{NavyBlue}\Delta_k^{p3}} &\leq \alpha \sum_{k=1}^{k_T}\sum_{s,a,s'} \nu_k(s,a) \Big|p_k(s'|s,a) -p(s'|s,a) \Big| \cdot \big|w_k^{s}(s')\big| \\
 &\leq 2 \alpha \sum_{k=1}^{k_T}\sum_{s,a} \nu_k(s,a) \sum_{s'}\beta_{p,k}^{sas'} \cdot \big|w_k^{s}(s')\big| \\
 &= 4 \alpha \sum_{k=1}^{k_T} \sum_{s,a} \nu_k(s,a) \Bigg[
 \sqrt{ \frac{\lnn{{6SAT}/{\delta} }}{\Np}} \sum_{s' \in \mathcal{S}} \sqrt{\wh{p}_k(s'|s,a)(1-\wh{p}_k(s'|s,a)) w_k^{s}(s')^2}\\&~~~~~~~~~~~~~~~~~~~~~~~~~~~~~~~~~~~~~~~~ +  \frac{3 \lnn{{6SAT}/{\delta}}}{\Np} \sum_{s'}\underbrace{\big|w_k^{sa}(s')\big|}_{\leq \diameter/\alpha}  \Bigg]
\end{align*}
We denote by $V_k(s,a) := \alpha^2\sum_{s'} \wh{p}_k(s'|s,a)w_k^{s}(s')^2$. 
We can prove the following inequality:
\begin{lemma}\label{lem:bound_support_2}
 It holds almost surely that for all $k\geq 1$ and for all $(s,a,s') \in \calS \times \A \times \calS$:
 \begin{align}
  \alpha \sum_{s' \in \mathcal{S}} \sqrt{\wh{p}_k(s'|s,a)(1-\wh{p}_k(s'|s,a)) w_k^{s}(s')^2} \leq \sqrt{V_k(s,a) \cdot \left(\nextstates(s,a) -1 \right)}
 \end{align}
\end{lemma}
\begin{proof}
        Define $\calS_k(s,a) = \{s' \in \calS \;:\; \wh{p}_k(s'|s,a) > 0\}$. 
        Then, using Cauchy-Schartz inequality we have 
\begin{align*}
 \sum_{s' \in {\color{blue}\mathcal{S}}} &\sqrt{\wh{p}_k(s'|s,a)(1-\wh{p}_k(s'|s,a)) w_k(s')^2} = \sum_{s' \in {\color{red}\mathcal{S}_k(s,a)}} \sqrt{\wh{p}_k(s'|s,a)(1-\wh{p}_k(s'|s,a)) w_k(s')^2}\\
 &\leq  \sqrt{\Bigg(\sum_{s' \in {\color{red}\mathcal{S}_k(s,a)}}1-\wh{p}_k(s'|s,a)\Bigg) \cdot \Bigg(\sum_{s' \in {\color{red}\mathcal{S}_k(s,a)}} \wh{p}_k(s'|s,a) w_k(s')^2 \Bigg)}\\
 &= \sqrt{\bigg(\nextstates_k(s,a) - 1\bigg) \cdot \Bigg(\sum_{s' \in {\color{blue}\mathcal{S}}} \wh{p}_k(s'|s,a) w_k(s')^2 \Bigg)} \leq \sqrt{\nextstates(s,a)  \sum_{s' \in {\color{blue}\mathcal{S}}} \wh{p}_k(s'|s,a) w_k(s')^2 }
\end{align*}
By definition, for all $s'\in\calS$, $w_k(s') = h_k(s') - \E_{X \sim \wh{p}_k(\cdot|s,a)}[h_k(X)]$ and so
\[ \sum_{s' \in \calS} \wh{p}_k(s'|s,a) w_k(s')^2 = \Var{h_k}{\wh{p}_k(\cdot|s,a)}\]
\end{proof}
As a consequence of Lem.~\ref{lem:bound_support_2},

\begin{align*}
 \sum_{k=1}^{k_T} {\color{NavyBlue}\Delta_k^{p3}} &\leq 4 \sum_{k=1}^{k_T} \sum_{s,a} \nu_k(s,a) \Bigg[
 \sqrt{ V_k(s,a) \frac{\nextstates(s,a)}{\Np}  \lnn{\frac{6SAT}{\delta} }} +  \frac{3 \diameter S }{\Np} \lnn{\frac{6SAT}{\delta} } \Bigg]\\
 &= 4 \sum_{k=1}^{k_T} \sum_{t=t_k}^{t_{k+1}-1} \Bigg[
 \sqrt{ V_k(s_t,a_t) \frac{\nextstates(s_t,a_t)}{\Np[s_t,a_t]}  \lnn{\frac{6SAT}{\delta} }} +  \frac{3 \diameter S }{\Np[s_t,a_t]} \lnn{\frac{6SAT}{\delta} } \Bigg].
\end{align*}

Applying Cauchy-Schwartz gives

\begin{align*}
 \sum_{k=1}^{k_T} \sum_{t=t_k}^{t_{k+1}-1} 
 \sqrt{ V_k(s_t,a_t)} \sqrt{ \frac{\nextstates(s_t,a_t)}{\Np[s_t,a_t]} } \leq& \sqrt{\sum_{k=1}^{k_T} \sum_{t=t_k}^{t_{k+1}-1} \frac{\nextstates(s_t,a_t)}{\Np[s_t,a_t]} \sum_{k=1}^{k_T} \sum_{t=t_k}^{t_{k+1}-1} V_k(s_t,a_t)}\\
 &= \sqrt{\sum_{k=1}^{k_T} \sum_{s,a} \frac{\nextstates(s,a) \nu_k(s,a)}{\Np[s,a]} \sum_{t=1}^{T} V_{k_t}(s_t,a_t)}.
\end{align*}

Using Lem.~\ref{lem:sqrt_log}, Jensen's inequality and the fact that $N_{k_T+1}^+(s,a)\leq T$ (as in Sec.~\ref{chap:ucrlb:sec:summing.episodes}), we can bound the first sum

\begin{align*}
 \sum_{s,a} \sum_{k=1}^{k_T}  \frac{\nextstates(s,a) \nu_k(s,a)}{\Np[s,a]} &\leq 2 \sum_{s,a}\nextstates(s,a)\left( 1+\lnn{N_{k_T+1}^+(s,a)}\right)\\
 &\leq 2 \left(1+\lnn{\frac{\sum_{s,a} \nextstates(s,a) N_{k_T+1}^+(s,a)}{\sum_{s,a} \nextstates(s,a)}}\right)\sum_{s,a}\nextstates(s,a)\\
 &\leq 2(1+\lnn{T})\sum_{s,a}\nextstates(s,a).
\end{align*}
To bound the second sum $\sum_{t=1}^{T} V_{k_t}(s_t,a_t)$, we rely on the following Lemma:
\begin{lemma}\label{lem:mds_variances}
Under event $E$, with probability at least $1 - \frac{\delta}{6}$:
\begin{align}
 \forall T\geq 1,~~\sum_{t=1}^{T} V_{k_t}(s_t,a_t) \leq \sum_{t=1}^{T} \Var{\alpha h_{k_t}}{\wh{p}_{k_t}(\cdot|s_t)} + 2\diameter^2 \sqrt{ T \lnn{\frac{4T}{\delta}}}
\end{align}
\end{lemma}
\begin{proof}
 We notice that for all $k\geq 1$ and $s\in \calS$, $\sum_a \pi_k(s,a) V_k(s,a) = \Var{ \alpha h_{k}}{\wh{p}_{k}(\cdot|s)}$. The concentration inequality then follows from a martingale argument and Azuma's inequality.
\end{proof}
From Lem.~\ref{lem:mds_variances} it follows that
\begin{align}\label{eq:main.variance}
 \sum_{k=1}^{k_T} {\color{NavyBlue}\Delta_k^{p3}} \leq &4 
 \sqrt{ 2\Big(1+\ln(T)\Big) \lnn{\frac{6SAT}{\delta} } \left(\sum_{s,a}\nextstates(s,a)\right)\left(\diameter^2 \sqrt{2 T \lnn{\frac{T}{\delta}}} + \sum_{t=1}^T\Var{\alpha h_{k_t}}{\wh{p}_{k_t}(\cdot|s_t)} \right) } \nonumber\\ &+  {24 \diameter S^2 A } \lnn{\frac{6SAT}{\delta} }(1+\ln(T))
\end{align}
It now remains to bound $\sum_{k=1}^{k_T}{\color{RedViolet}\Delta_k^{p2}}$. As shown by~\citep{Jaksch10,fruit2018constrained} using telescopic sum argument: $\sum_{k=1}^{k_T}{\color{RedViolet}\Delta_k^{p2}} \leq \sum_{k=1}^{k_T}{\color{Aquamarine}\Delta_k^{p4}} + \diameter k_T$ where
\[
        {\color{Aquamarine}\Delta_k^{p4}} =\alpha \sum_{t=t_k}^{t_{k+1}-1} \left(
                \sum_{a,s'} \pi_k(s_t,a)p(s'|s,a)w_k(s') - w_k(s_{t+1})
        \right)
\]
We bound $\sum_{k=1}^{k_T}{\color{Aquamarine}\Delta_k^{p4}}$  using Freedman's inequality instead of Azuma's.
\begin{lemma}\label{lem:mds_delta_p4_2}
 Under event $E$, with probability at least $1-\frac{\delta}{6}$:
 \begin{align}\label{eq:bound.c}
  \forall T \geq 1, ~ \sum_{k=1}^{k_T}{\color{Aquamarine}\Delta_k^{p4}} \leq 2 \sqrt{ \left( \sum_{t =1}^{T} \Var{\alpha h_k}{\wb{p}_{k_t}(\cdot|s_t)} \right) \cdot\lnn{ \frac{24T}{\delta} }  } + 4 \rmaxbound D \lnn{\frac{24T}{\delta}}
 \end{align}
\end{lemma}
\begin{proof}
 We use a martingale argument and Prop.~\ref{prop:fi} (see App.~\ref{app:mds_proofp4_2} for further details).
\end{proof}

\subsection{From $D$ to $\sqrt{D}$: Bounding the sum of variances}\label{chap:ucrlb:sec:bound.sum.variances}

The main terms appearing respectively in \eqref{eq:mds_delta_p1_2}, \eqref{eq:main.variance} and \eqref{eq:bound.c} all have the form of a \emph{sum of variances over time} $\sum_{t=1}^T \Var{\alpha h_{k_t}}{p_t}$ with $p_t$ a distribution over states (respectively $p_{k_t}(\cdot|s_t)$, $\wb{p}_{k_t}(\cdot|s_t)$ and $\wh{p}_{k_t}(\cdot|s_t)$), and $h_{k_t}$ the optimistic bias of episode $k_t$. A first \emph{na\"ive} upper bound of this sum can be derived using Popoviciu's inequality that we recall in Prop.~\ref{prop:popoviciu}. 
\begin{proposition}[Popoviciu's inequality on variances]\label{prop:popoviciu}
        Let $M$ and $m$ be upper and lower bounds on the values of a random variable $X$ \ie $\Proba{m \leq X \leq M}=1$. Then $\V(X) \leq \frac{1}{4} (M-m)^2$.
\end{proposition}
Using Popoviciu's inequality and under event $E$, \[\Var{\alpha h_{k_t}}{p_t} \leq \SP{\alpha h_k}^2/4 =\alpha^2 \SP{h_k}^2/4 \leq \diameter^2/4\] and so $\sum_{t=1}^T \Var{\alpha h_{k_t}}{p_t} \leq \diameter^2 T/4$. Unfortunately, this would result in a regret bound scaling as $\wt{\bigO}(\diameter \sqrt{T})$ (ignoring all other terms like $S$, $A$, logarithmic terms, etc.) which is \emph{not better} than the classical bound of \ucrl. In this section, we show that the cumulative sum of variances only scales as $\wt{\bigO}(\diameter T + \diameter^2 \sqrt{T})$ resulting in a regret bound of order $\wt{\bigO}\pare{\sqrt{\diameter T} + \diameter T^{1/4}}$ (ignoring all other terms).

We start by analyzing the variance term $\Var{\alpha \h}{\ph{s_t}{\cdot}}$. The other variance terms $\Var{\alpha \h}{p_{k}(\cdot|s_t)}$ and $\Var{\alpha  \h}{\wb{p}_{k}(\cdot|s_t)}$ can be addressed in the same way. We do the following decomposition:

\begin{align*}
 \Var{\alpha \h}{\ph{s_t}{\cdot}} &= \alpha^2\left(\ph{s_t}{\cdot}\T\h^2 - \pare{\ph{s_t}{\cdot}\T\h}^2\right) \\
                                  &= \alpha^2\Big(\underbrace{\pare{\ph{s_t}{\cdot} - \p{s_t}{\cdot}}\T\h^2}_{\textcircled{1}} + \underbrace{\p{s_t}{\cdot}\T\h^2 - \h^2(s_{t+1})}_{\textcircled{2}} +\underbrace{\h^2(s_{t+1}) - \pare{\ph{s_t}{\cdot}\T\h}^2}_{\textcircled{3}}\Big)
\end{align*}
Recall thet, $p_{k}(s'|s):= \sum_{a \in \A_{s} }\pi_{k}(s,a)p_{k}(s'|s,a)$, $\p{s}{s'}:= \sum_{a \in \A_{s} }\pi_{k}(s,a)p(s'|s,a)$ and $\wh{p}_k(s'|s):= \sum_{a \in \A_{s} }\pi_{k}(s,a)\wh{p}_k(s'|s,a)$, for every $s,s' \in \calS$ and every $k\geq 1$.

Notice that for any \rv $X$ and any scalar $a\in \Re$, $\V(X+a) = \V(X)$. Thus, the term $\Var{\alpha \h}{\ph{s_t}{\cdot}}$ remains unchanged when $h_k$ is shifted by an arbitrary constant vector \ie when $h_k$ is replaced by $w_k := h_k +\lambda_k e$. As in \ucrl, we minimize the $\ell_\infty$-norm of $w_k$ by choosing $\lambda_k= -\frac{1}{2}\left(\max_{s\in \calS}h_k(s) + \min_{s\in \calS}h_k(s) \right)$. We recall that under event $E$, $\|w_k\|_{\infty} \leq \diameter/(2\alpha)$ and so $\|w_k^2\|_{\infty} \leq \diameter^2/(4\alpha^2)$ \vspace{0.2cm}

$\textcircled{1}$ The \emph{first term} $\alpha^2\sum_{k=1}^{k_T}\sum_{t=t_k}^{t_{k+1}-1} \pare{\ph{s_t}{\cdot} - \p{s_t}{\cdot}}\T w_k^2$ is similar to  $\sum_{k=1}^{k_T}{\color{Mahogany}\Delta_k^{p1}}$ except that $\alpha w_k$ is replaced by $\alpha^2 w_k^2$ and $\pt{s_t}{\cdot}$ is replaced by $\ph{s_t}{\cdot}$. In the regret proof of \ucrl we have to decompose $\pt{s_t}{\cdot} - \p{s_t}{\cdot}$ into the sum of $\pt{s_t}{\cdot} - \ph{s_t}{\cdot}$ and $\ph{s_t}{\cdot} - \p{s_t}{\cdot}$. Here we no longer need this decomposition and we can use the same derivation with $\SP{\alpha^2 w_k^2} \leq \diameter^2/4$ instead. Therefore, with probability at least  $1-\frac{\delta}{6}$ (and under event $E$):
\begin{align*}
        \alpha^2\sum_{k=1}^{k_T} \sum_{t=t_k}^{t_{k+1}-1} &\pare{\ph{s_t}{\cdot} - \p{s_t}{\cdot}}\T w_k^2 \leq \frac{3}{2} \diameter^2 \sqrt{ \left(\sum_{s,a}\nextstates(s,a)\right) T \lnn{\frac{6SAT}{\delta}}}\\
 &+ \diameter^2\sqrt{T\ln \left(\frac{5T}{\delta}\right)}
 + 3 \diameter^2 S^2 A\ln\left(\frac{6SAT}{\delta} \right) (1 + \lnn{T})
\end{align*}

$\textcircled{2}$ The \emph{second term} $\alpha^2\sum_{k=1}^{k_T}\sum_{t=t_k}^{t_{k+1}-1} {\p{s_t}{\cdot}\T w_k^2 - w_k^2(s_{t+1})}$ is identical to the term bounded in \ucrl except that $\alpha w_k$ is replaced by $\alpha^2 w_k^2$. With probability at least $1-\frac{\delta}{6}$ (and under event $E$)~\citep[see \eg][]{fruit2019thesis}:
\begin{align*}
 \alpha^2\sum_{k=1}^{k_T}\sum_{t=t_k}^{t_{k+1}-1} {\p{s_t}{\cdot}\T w_k^2 - w_k^2(s_{t+1})} \leq \frac{\diameter^2}{2} \sqrt{ T \lnn{\frac{5T}{\delta}}}
\end{align*}

$\textcircled{3}$ The \emph{last term} $\alpha^2\sum_{k=1}^{k_T}\sum_{t=t_k}^{t_{k+1}-1} {w_k^2(s_{t+1}) - \pare{\ph{s_t}{\cdot}\T w_k}^2}$ is the \emph{dominant} one and requires more work. Unlike the first two terms, it scales \emph{linearly} with $T$ (instead of $\wt{\bigO}(\sqrt{T})$).
We first notice that
$
 \ph{s_t}{\cdot}\T w_k = w_k(s_t) + \ph{s_t}{\cdot}\T w_k - w_k(s_t)
$.
Using the fact that $(a+b)^2 = a^2 + b(2a +b)$ with $a = w_k(s_t)$ and $b = \ph{s_t}{\cdot}\T w_k - w_k(s_t) $ (and therefore $2a + b = w_k(s_t) + \ph{s_t}{\cdot}\T w_k$) we obtain:
\begin{align*}
 \pare{\ph{s_t}{\cdot}\T w_k}^2 = w_k^2(s_t) + \pare{\ph{s_t}{\cdot}\T w_k - w_k(s_t)}\cdot \pare{w_k(s_t) + \ph{s_t}{\cdot}\T w_k}
\end{align*}
and so applying the \emph{reverse triangle inequality}:
\begin{align}\label{eq:reverse_triangle}
 \pare{\ph{s_t}{\cdot}\T w_k}^2 \geq w_k^2(s_t) - \abs{\ph{s_t}{\cdot}\T w_k - w_k (s_t)} \cdot\abs{w_k(s_t) + \ph{s_t}{\cdot}\T w_k}
\end{align}
For all $k\geq 1$ and $s\in \calS$, we define $r_k(s) := \sum_{a}\pi_k(a|s)r_k(s,a)$.
Using the (near-)optimality equation we can write: \begin{align*}\abs{g_k - \rt{s_t} + \alpha \big(w_k(s_t) - \pt{s_t}{\cdot}\T w_k\big) } = \abs{ g_k - \rt{s_t} + \alpha \big(h_k(s_t) - \pt{s_t}{\cdot}\T h_k\big) } \leq \varepsilon_k \end{align*}
Moreover, $\varepsilon_k = \frac{\rmaxbound}{t_k} \leq \rmaxbound$.
As a result, since $\alpha >0$:
\begin{align*}
 \alpha \big|&\ph{s_t}{\cdot}\T w_k - w_k(s_t)\big|\\
 &= \abs{g_k - \rt{s_t} + \alpha\big(w_k(s_t) - \pt{s_t}{\cdot}\T w_k\big)  -g_k + \rt{s_t}  + \alpha\pare{\pt{s_t}{\cdot} - \ph{s_t}{\cdot}}\T w_k }\\
 &\leq \underbrace{\abs{ g_k - \rt{s_t} + \alpha \big(w_k(s_t) - \pt{s_t}{\cdot}\T w_k\big) }}_{\leq \rmaxbound} + \underbrace{\abs{\rt{s_t} - g_k}}_{\leq \rmaxbound} + \alpha\abs{\pare{\pt{s_t}{\cdot} - \ph{s_t}{\cdot}}\T w_k } \\
 &\leq 2\rmaxbound + \alpha \abs{\pare{\pt{s_t}{\cdot} - \ph{s_t}{\cdot}}\T w_k } 
\end{align*}
It is also immediate to see that $\left|w_k(s_t) + \ph{s_t}{\cdot}\T w_k \right| \leq 2 \|w_k\|_{\infty} \leq \diameter/\alpha$. Plugging these inequalities into~\eqref{eq:reverse_triangle} and adding $w_k^2(s_{t+1})$ we obtain:
\begin{align}\label{eq:variance_main}
\begin{split}
 \alpha^2\left(w_k^2(s_{t+1}) - \pare{\ph{s_t}{\cdot}\T w_k}^2\right) \leq  &\pare{2 \rmaxbound + \alpha \abs{\pare{\pt{s_t}{\cdot} - \ph{s_t}{\cdot}}\T w_k }} \diameter\\ 
 &+ \alpha^2\left(w_k^2(s_{t+1}) - w_k^2(s_t)\right)
 \end{split}
\end{align}
It is easy to bound the telescopic sum 

\begin{align}\label{eq:telescopic2}
\alpha^2\sum_{t=t_k}^{t_{k+1}-1} w_k^2(s_{t+1}) - w_k^2(s_t) = \alpha^2\pare{w_k^2(s_{t_{k+1}}) - w_k^2(s_{t_{k}})} \leq \alpha^2 w_k^2(s_{t_{k+1}}) \leq \diameter^2/4
\end{align} 

Finally, the sum $\alpha \sum_{k=1}^{k_T} \sum_{t=t_k}^{t_{k+1} - 1}\abs{\pare{\pt{s_t}{\cdot} - \ph{s_t}{\cdot}}\T w_k }$ can be bounded in the exact same way as $\sum_{k=1}^{k_T}{\color{Mahogany}\Delta_k^{p1}}$ (see Sec.~\ref{chap:ucrlb:sec:bound.trans.proba}). With probability at least $1-\frac{\delta}{6}$:

\begin{align}\label{eq:bound_delta_p1_like}
 \alpha \sum_{k=1}^{k_T} \sum_{t=t_k}^{t_{k+1} - 1}\abs{\pare{\pt{s_t}{\cdot} - \ph{s_t}{\cdot}}\T w_k } \leq &3 \diameter \sqrt{ \left(\sum_{s,a}\nextstates(s,a)\right) T \lnn{\frac{6SAT}{\delta}}} + 4\diameter\sqrt{T\ln \left(\frac{5T}{\delta}\right)}\nonumber\\
 &+ 6 \diameter S^2 A\ln\left(\frac{6SAT}{\delta} \right) (1 + \lnn{T})
\end{align}

After gathering \eqref{eq:telescopic2} and \eqref{eq:bound_delta_p1_like} into \eqref{eq:variance_main}) we conclude that with probability at least $1- \frac{\delta}{6}$ (and under event $E$):

\begin{align*}
 \alpha^2\sum_{k=1}^{k_T}\sum_{t=t_k}^{t_{k+1}-1} {w_k^2(s_{t+1}) - \pare{\ph{s_t}{\cdot}\T w_k}^2} \leq \underbrace{2 \rmaxbound^2 D T }_{\text{main term}} + \frac{k_T\diameter^2}{4} + \wt{\bigO}\pare{\diameter^2 \sqrt{\left(\sum_{s,a}\nextstates(s,a)\right) T }}
\end{align*}

In conclusion, there exists an \emph{absolute} numerical constant $\beta >0$ (\ie independent of the MDP instance) such that with probability at least $1-\frac{5\delta}{6}$:

\begin{align*}
 \sum_{t=1}^{T}\Var{\alpha h_{k_t}}{\wh{p}_{k_t}(\cdot|s_t)} \leq \beta \cdot\left( \rmaxbound^2 D T + \diameter^2 \sqrt{\left(\sum_{s,a} \nextstates(s,a)\right) T\lnn{\frac{T}{\delta}}} + \diameter^2 S^2 A \lnn{\frac{T}{\delta}}\lnn{T} \right).
\end{align*}
We can prove the same bound (possibly with a different multiplicative constant $\beta$) for $\sum_{t=1}^{T}\Var{\alpha h_{k_t}}{\wb{p}_{k_t}(\cdot|s_t)}$ and $\sum_{t=1}^{T}\Var{\alpha h_{k_t}}{{p}_{k_t}(\cdot|s_t)}$ using the same derivation.

\subsection{Completing the regret bound of Thm.~\ref{thm:regret.bound2}}

After plugging the bound derived for the sum of variances in the previous section (Sec.~\ref{chap:ucrlb:sec:bound.sum.variances}) into~\eqref{eq:mds_delta_p1_2},~\eqref{eq:main.variance} and~\eqref{eq:bound.c}, we notice that~\eqref{eq:mds_delta_p1_2} and~\eqref{eq:bound.c} can be upper-bounded by \eqref{eq:main.variance} \emph{up to a multiplicative numerical constant} ans so it is enough to restrict attention to \eqref{eq:main.variance}. The dominant term that we obtain is (ignoring numerical constants):

\begin{align*}
 \rmaxbound \sqrt{ \left(\sum_{s,a} \nextstates(s,a)\right) \lnn{\frac{T}{\delta}}\lnn{T} \left(D T + D^2 \sqrt{\left(\sum_{s,a} \nextstates(s,a)\right) T \lnn{\frac{T}{\delta}}} + D^2 S^2 A \lnn{\frac{T}{\delta}}\lnn{T} \right)}
\end{align*}

Using the fact that $\sqrt{\sum_{i} a_i} \leq \sum_i\sqrt{a_i}$ for any $a_i\geq 0$, we can bound the above square-root term by the sum of three simpler terms:
\begin{flalign*}
&\text{(1) A $\sqrt{T}$-term (dominant): }~\rmaxbound \sqrt{ D \left(\sum_{s,a} \nextstates(s,a)\right) T \lnn{\frac{T}{\delta}}\lnn{T} } &\\
&\text{(2) A $T^{1/4}$-term: }~ \diameter \left(\sum_{s,a} \nextstates(s,a)\right)^{3/4} T^{1/4} \left(\lnn{\frac{T}{\delta}}\right)^{3/4}\sqrt{\lnn{T}}&\\
&\text{(3) A logarithmic term: }~ \diameter \sqrt{ S^2 A \left(\sum_{s,a} \nextstates(s,a)\right) } \lnn{\frac{T}{\delta}}\lnn{T} \leq \diameter  S^2 A \lnn{\frac{T}{\delta}}\lnn{T}&
\end{flalign*}
When $  T \geq D^2 \left(\sum_{s,a} \nextstates(s,a)\right) \lnn{\frac{T}{\delta}}$, we notice that the $T^{1/4}$-term (2) is actually upper-bounded by the $\sqrt{T}$-term (1), while for $T \leq D^2 \left(\sum_{s,a} \nextstates(s,a)\right) \lnn{\frac{T}{\delta}}$ we can use the following trivial upper-bound $\rmaxbound T$ on the regret:

\begin{align*}
 R(T, M^\star,\text{\ucrlb})\leq \rmaxbound T  \leq  \rmaxbound D^2 \left(\sum_{s,a} \nextstates(s,a)\right) \lnn{\frac{T}{\delta}} \leq \rmaxbound D^2 S^2 A \lnn{\frac{T}{\delta}}
\end{align*}

To complete the regret bound of Thm.~\ref{thm:regret.bound2} we also need to take into consideration~\eqref{eqn:splitting} and \eqref{eqn:bound_reward} as well as the \emph{lower order terms} of \eqref{eq:mds_delta_p1_2}, \eqref{eq:main.variance} and \eqref{eq:bound.c}. It turns out that the only terms that are not already upper-bounded by (1), (2) and (3) (up to multiplicative numerical constants) sum as: \[\rmaxbound \sqrt{SAT \lnn{\frac{T}{\delta}}} +\rmaxbound SA\lnn{\frac{T}{\delta}} \lnn{T} + \diameter S^2A\lnn{\frac{T}{\delta}} \lnn{T}\]
All the above logarithmic terms can be bounded by: $\max \left\{\rmaxbound, \rmaxbound D^2 \right\} S^2 A  \lnn{\frac{T}{\delta}}\lnn{T}$. Moreover, all the $\sqrt{T}$-terms can be bounded by \[\max \left\{\rmaxbound, \rmaxbound \sqrt{D} \right\} \sqrt{\left(\sum_{s,a} \nextstates(s,a)\right) T \lnn{\frac{T}{\delta}}\lnn{T}}\]
To conclude, we only need to \emph{adjust} $\delta$ to obtain an event of probability at least $1-\delta$. This will \emph{only} impact the multiplicative numerical constants of the above terms.

\bibliographystyle{apalike}
\bibliography{biblio}

\begin{thebibliography}{}

\bibitem[Audibert et~al., 2007]{audibert2007tuning}
Audibert, J.-Y., Munos, R., and Szepesv{\'a}ri, C. (2007).
\newblock Tuning bandit algorithms in stochastic environments.
\newblock In {\em Algorithmic Learning Theory}, pages 150--165, Berlin,
  Heidelberg. Springer Berlin Heidelberg.

\bibitem[Audibert et~al., 2009]{Audibert:2009:ETU:1519541.1519712}
Audibert, J.-Y., Munos, R., and Szepesv\'{a}ri, C. (2009).
\newblock Exploration-exploitation tradeoff using variance estimates in
  multi-armed bandits.
\newblock {\em Theor. Comput. Sci.}, 410(19):1876--1902.

\bibitem[Azar et~al., 2013]{AzarMK13}
Azar, M.~G., Munos, R., and Kappen, H.~J. (2013).
\newblock Minimax {PAC} bounds on the sample complexity of reinforcement
  learning with a generative model.
\newblock {\em Mach. Learn.}, 91(3):325--349.

\bibitem[Azar et~al., 2017]{pmlr-v70-azar17a}
Azar, M.~G., Osband, I., and Munos, R. (2017).
\newblock Minimax regret bounds for reinforcement learning.
\newblock In {\em Proceedings of the 34th International Conference on Machine
  Learning}, volume~70 of {\em Proceedings of Machine Learning Research}, pages
  263--272, International Convention Centre, Sydney, Australia. PMLR.

\bibitem[Freedman, 1975]{freedman1975}
Freedman, D.~A. (1975).
\newblock On tail probabilities for martingales.
\newblock {\em Ann. Probab.}, 3(1):100--118.

\bibitem[Fruit, 2019]{fruit2019thesis}
Fruit, R. (2019).
\newblock {\em {Exploration-exploitation dilemma in Reinforcement Learning
  under various form of prior knowledge}}.
\newblock Theses, {Universit{\'e} de Lille 1, Sciences et Technologies; CRIStAL
  UMR 9189}.

\bibitem[Fruit et~al., 2018]{fruit2018constrained}
Fruit, R., Pirotta, M., Lazaric, A., and Ortner, R. (2018).
\newblock Efficient bias-span-constrained exploration-exploitation in
  reinforcement learning.
\newblock {\em CoRR}, abs/1802.04020.

\bibitem[Jaksch et~al., 2010]{Jaksch10}
Jaksch, T., Ortner, R., and Auer, P. (2010).
\newblock Near-optimal regret bounds for reinforcement learning.
\newblock {\em Journal of Machine Learning Research}, 11:1563--1600.

\bibitem[Lattimore and Hutter, 2012]{Lattimore12pacbounds}
Lattimore, T. and Hutter, M. (2012).
\newblock Pac bounds for discounted mdps.
\newblock In {\em In Proc. 23rd International Conf. on Algorithmic Learning
  Theory (ALT’12), volume 7568 of LNAI}. Springer.

\bibitem[Lattimore and Hutter, 2014]{Lattimorenearoptimalpac}
Lattimore, T. and Hutter, M. (2014).
\newblock Near-optimal pac bounds for discounted mdps.
\newblock {\em Theoretical Computer Science}, 558:125--143.

\bibitem[Lattimore and Szepesv\'{a}ri, 2018]{bandittorcsaba}
Lattimore, T. and Szepesv\'{a}ri, C. (2018).
\newblock Bandit algorithms.
\newblock Pre-publication version.

\bibitem[Maillard et~al., 2014]{MaiManMan14}
Maillard, O.-A., Mann, T.~A., and Mannor, S. (2014).
\newblock How hard is my mdp?” the distribution-norm to the rescue”.
\newblock In Ghahramani, Z., Welling, M., Cortes, C., Lawrence, N., and
  Weinberger, K., editors, {\em Advances in Neural Information Processing
  Systems 27}, page 1835–1843. Curran Associates, Inc.

\bibitem[Munos and Moore, 1999]{Munos99influenceand}
Munos, R. and Moore, A. (1999).
\newblock Influence and variance of a markov chain: Application to adaptive
  discretization in optimal control.
\newblock In {\em Proceedings: International Astronomical Union Transactions,
  v. 16B p}, pages 355--362.

\bibitem[Puterman, 1994]{puterman1994markov}
Puterman, M.~L. (1994).
\newblock {\em Markov Decision Processes: Discrete Stochastic Dynamic
  Programming}.
\newblock John Wiley \& Sons, Inc., New York, NY, USA.

\bibitem[Talebi and Maillard, 2018]{TalebiKLUCRL}
Talebi, M.~S. and Maillard, O. (2018).
\newblock Variance-aware regret bounds for undiscounted reinforcement learning
  in mdps.
\newblock In {\em {ALT}}, volume~83 of {\em Proceedings of Machine Learning
  Research}, pages 770--805. {PMLR}.

\end{thebibliography}

\appendix
\section{Additional Results}
\begin{lemma}\label{lem:sqrt_log}
 It holds almost surely that for all $k\geq 1$ and for all $(s,a) \in \calS \times \A \times \calS$:
 \begin{align}
 {\color{OliveGreen}\sum_{k=1}^{k_T} \frac{\nu_k(s,a)}{\sqrt{\Np}}} \leq 3 \sqrt{N_{k_T+1}(s,a)} ~~ \text{and} ~~
  {\color{Bittersweet}\sum_{k=1}^{k_T} \frac{\nu_k(s,a)}{\Np}} \leq 2 + 2\lnn{N_{k_T+1}^+(s,a)}
 \end{align}
\end{lemma}
\begin{proof}
 The proof follows from the rate of divergence of the series $\sum_{i=1}^n \frac{1}{\sqrt{i}} \sim \sqrt{n}$ and $\sum_{i=1}^n \frac{1}{i} \sim \lnn{n}$ respectively when $n\to +\infty$.
\end{proof}

\section{MDS}
For any $t \geq 0$, the \sigalg induced by the past history of state-action pairs and rewards up to time $t$ (included) is denoted $\mathcal{F}_t = \sigma(s_1, a_1,r_1, \dots, s_t,a_t, r_t,s_{t+1})$ where by convention $\mathcal{F}_0 = \sigma \left( \emptyset \right)$ and $\mathcal{F}_\infty := \cup_{t\geq 0}\mathcal{F}_{t}$. Trivially, for all $t\geq 0$, $\mathcal{F}_{t} \subseteq \mathcal{F}_{t+1}$ and the filtration $\left(\mathcal{F}_t \right)_{t \geq 0}$ is denoted by $\Fil$. We recall that $k_t$ is the integer-valued \rv indexing the current episode at time $t$. It is immediate from the termination condition of episodes that for all $t\geq 1$, $k_t$ is $\F_{t-1}$-measurable \ie the past sequence $(s_1, a_1,r_1, \dots, s_{t-1},a_{t-1}, r_{t-1},s_{t})$ fully determines the ongoing episode at time $t$. As a consequence, the stationary (randomized) policy $\pi_{k_t}$ executed at time $t$ is also $\F_{t-1}$-measurable.

\subsection{Proof of Lemma~\ref{lem:mds_delta_p4_2}}
\label{app:mds_proofp4_2}

Let's define the stochastic process \[X_t :=\sum_{{\color{blue}a}, s' }\pi_{k_t}(s_t,{\color{blue}a}) p_{k_t}(s'|s_t,{\color{blue}a})h_{k_t}(s') - \sum_{s'}p_{k_t}(s'|s_t,{\color{red}a_t})h_{k_t}(s')\]
Let's define $\lambda_{t} = - \sum_{{a}, s' }\pi_{k_t}(s_t,{a}) p_{k_t}(s'|s_t,{a})h_{k_t}(s')$ and $w_{t} =  h_{k_t} +\lambda_{t} e$. Since by definition $\sum_{s'}p_{k_t}(s'|s_t,a_t) = 1$, we have
\[ X_t = - \sum_{s'}p_{k_t}(s'|s_t,{\color{red}a_t})w_{t}(s')\]
It is easy to verify that $\Ex{X_t |\F_{t-1}} = 0$ and so $(X_t,\F_t)_{t\geq 1}$ is an MDS. Moreover, $|X_t| \leq \|w_{t}\|_{\infty} \leq \SP{h_{k_t}} \leq \diameter$ and 
\begin{align*}
 \Varcond{X_t}{\F_{t-1}} = \sum_{a }\pi_{k_t}(s_t,a)\left( \sum_{s'}p_{k_t}(s'|s_t,a)w_{t}(s') \right)^2
\end{align*}
\begin{proposition}\label{prop:sum.square}
 For any $n\geq 1$ and any $n$-tuple $(a_1,\dots,a_n) \in \Re^n$, $\left(\sum_{i=1}^n a_i\right)^2\leq n \left(\sum_{i=1}^n a_i^2 \right)$.
\end{proposition}
\begin{proof}
 The statement is trivially true for $n=1$. For $n=2$ we have $(a_1 -a_2)^2 = a_1^2 + a_2^2 - 2a_1 a_2 \geq 0$ implying that $2a_1 a_2 \leq a_1^2 + a_2^2$. Therefore, $(a_1 +a_2)^2 = a_1^2 + a_2^2 +2a_1 a_2 \leq 2(a_1^2 + a_2^2)$ and so the result holds. We prove the result for $n\geq2$ by induction. Assumed that it is true for any $n\geq 2$. Then we have:
 \begin{align*}
  \left(\sum_{i=1}^{n+1} a_i\right)^2 &= \underbrace{\left(\sum_{i=1}^{n} a_i\right)^2}_{\leq n \left(\sum_{i=1}^{n} a_i ^2\right)} + a_{n+1}^2 +  2 a_{n+1} \sum_{i=1}^{n} a_i \\
  &\leq n \left(\sum_{i=1}^{n} a_i ^2\right) + a_{n+1}^2 + \sum_{i=1}^{n} \underbrace{2 a_i a_{n+1}}_{\leq a_i^2 + a_{n+1}^2} 
  \leq (n+1)\cdot\left(\sum_{i=1}^{n+1} a_i^2\right)
 \end{align*}
where the first inequality follows from the induction hypothesis and the second inequality follows from the inequality for $n=2$ that we proved. This concludes the proof.
\end{proof}
For the sake of clarity we will now use the notation $p_{k}(s'|s):= \sum_{a \in \A_{s} }\pi_{k}(s,a)p_{k}(s'|s,a)$ for every $s,s' \in \calS$ and every $k\geq 1$.
Using Prop.~\ref{prop:sum.square} we have that 
\begin{align*}
 \Varcond{X_t}{\F_{t-1}} &\leq S \sum_{a,s'}\pi_{k_t}(s_t,a)\underbrace{p_{k_t}(s'|s_t,a)^2}_{\leq p_{k_t}(s'|s_t,a)}w_{k_t}(s')^2\\ 
 &\leq S \sum_{a,s'}\pi_{k_t}(s_t,a)p_{k_t}(s'|s_t,a)w_{k_t}(s')^2 = S\cdot \Var{h_{k_t}}{p_{k_t}(\cdot|s_t)}
\end{align*}
After applying Freedman's inequality (Prop.~\ref{prop:fi}) to the MDS $(X_t,\F_t)_{t\geq 1}$ we obtain that with probability at least $1-\frac{\delta}{6}$, for all $T\geq 1$:
\begin{align}\label{eq:variance.mds}
 \sum_{k=1}^{k_T} \sum_{s,a,s'} \nu_k(s) \pi_k(s,a) p_k(s'|s,a)h_k(s') \leq& \sum_{k=1}^{k_T} \sum_{s,a,s'} \nu_k(s,a) p_k(s'|s,a)h_k(s') + 2 \diameter \lnn{\frac{24T}{\delta}}\nonumber \\ 
 &+ 2 \sqrt{S \lnn{\frac{24T}{\delta}} \sum_{t=1}^{T}  \Var{h_{k_t}}{p_{k_t}(\cdot|s_t)} }
\end{align}
We can do exactly the same analysis with the stochastic process
\[X_t :=\sum_{{\color{blue}a}, s' }\pi_{k_t}(s_t,{\color{blue}a}) p(s'|s_t,{\color{blue}a})h_{k_t}(s') - \sum_{s'}p(s'|s_t,{\color{red}a_t})h_{k_t}(s')\]
\ie with $p$ instead of $p_{k_t}$ and we obtain that with probability at least $1-\frac{\delta}{6}$, for all $T\geq 1$:
\begin{align}\label{eq:variance.mds2}
 -\sum_{k=1}^{k_T} \sum_{s,a,s'} \nu_k(s) \pi_k(s,a) p(s'|s,a)h_k(s') \leq& -\sum_{k=1}^{k_T} \sum_{s,a,s'} \nu_k(s,a) p(s'|s,a)h_k(s') + 2 \diameter \lnn{\frac{24T}{\delta}}\nonumber \\ 
 &+ 2 \sqrt{S \lnn{\frac{24T}{\delta}} \sum_{t=1}^{T}  \Var{h_{k_t}}{\wb{p}_{k_t}(\cdot|s_t)} }
\end{align}
with the notation $\wb{p}_k(s'|s):= \sum_{a \in \A_{s} }\pi_{k}(s,a)p(s'|s,a)$ for every $s,s' \in \calS$ and $k\geq 1$.

\subsection{Definition of The Confidence Intervalsd}
\label{app:conf_interval}
\begin{theorem}
        The probability that there exists $k\geq 1$ \st the true MDP $M$ does not belong to the extended MDP ${\mathcal{M}}_k$ defined by Eq.~\ref{eq:confidence_interval_p} and \ref{eq:confidence_interval_r} is at most $\frac{\delta}{3}$, that is \[\Proba{ \exists k \geq 1, \text{ \st~} M \not\in {\mathcal{M}}_k } \leq \frac{\delta}{3}.\]
\end{theorem}
\begin{proof}
We want to bound the probability of event $E := \bigcup_{k=1}^{+\infty} \left\{M \not\in {\mathcal{M}}_k \right\}$.
As explained by \citet[Section 4.4]{bandittorcsaba}, when $(s,a)$ is visited for the $n$-th times, the reward that we observe is the $n$-th element of an infinite sequence of \iid \rv lying in $[0,\rmaxbound]$ with expected value $r(s,a)$. Similarly, the next state that we observe is the $n$-th element of an infinite sequence of \iid \rv lying in $\calS$ with probability density function (pdf) $p(\cdot|s,a)$.
In \ucrl, we defined the sample means $\wh{p}_k$ and $\wh{r}_k$, and the confidence intervals $B_p^k$ and $B_r^k$ (Eq.~\ref{eq:confidence_interval_p} and \ref{eq:confidence_interval_r}) as depending on $k$. Actually, this quantities depends only on the first $N_k(s,a)$ elements of the infinite \iid sequences that we just mentioned. For the rest of the proof, we will therefore slightly change our notations and denote by $\wh{p}_n(s'|s,a)$, $\wh{r}_n(s,a)$, $B_p^n(s'|s,a)$ and $B_r^n(s,a)$ the sample means and confidence intervals after the first $n$ visits in $(s,a)$.
Thus, the \rv that we denoted by $\wh{p}_k$ in \ucrl actually corresponds to $\wh{p}_{N_k(s,a)}$ with our new notation (and similarly for $\wh{r}_k$, $B_p^k$ and $B_r^k$). This change of notation will make the proof easier.

$M \not\in \M_k$ means that there exists $k\geq 1$ \st either $p(s'|s,a) \not\in B_p^{N_k(s,a)}(s,a,s')$ or $r(s,a) \not\in B_r^{N_k(s,a)}(s,a)$ for at least one $(s,a,s') \in \calS \times \A \times \calS$. This means that there exists at least one value $n\geq 0$ \st either $p(s'|s,a) \not\in B_p^{n}(s,a,s')$ or $r(s,a) \not\in B_r^{n}(s,a)$. As a consequence we have the following inclusion
\begin{align}\label{eq:inclusion_decomposition}
 E \subseteq \bigcup_{s,a} \bigcup_{n=0}^{+\infty} \left\{ r(s,a) \not\in B_r^{n}(s,a) \right\} \cup \bigcup_{s'} \left\{ p(s'|s,a) \not\in B_p^{n}(s,a,s') \right\}
\end{align}
Using Boole's inequality we thus have:
\begin{align}\label{eq:inclusion_decomposition_proba}
 \Proba{E} \leq \sum_{s,a} \sum_{n=0}^{+\infty} \left( \Proba{r(s,a) \not\in B_r^n(s,a)} + \sum_{s'} \Proba{p(s'|s,a) \not\in B_p^n(s,a,s') } \right)
\end{align}
Let's fix a 3-tuple $(s,a,s') \in \calS \times \A \times \calS$ and define for all $n\geq 0$
 \begin{align}
         \epsilon_{p,n}^{sas'} := \wh{\sigma}_{p,n}(s'|s,a)  \sqrt{\frac{2 \lnn{30 S^2A (n^+)^2/\delta} }{n^+}} + \frac{3 \lnn{30 S^2A (n^+)^2/\delta}}{n^+}\label{eq:epsilon.p.optimism}\\
         \epsilon_{r,n}^{sa} := \wh{\sigma}_{r,n}(s,a)  \sqrt{\frac{2 \lnn{30 SA(n^+)^2/\delta} }{n^+}} + \frac{3 \rmaxbound \lnn{30 SA (n^+)^2/\delta}}{n^+}\label{eq:epsilon.r.optimism}
 \end{align}
 where $\wh{\sigma}_{p,n}(s'|s,a)$ and $\wh{\sigma}_{r,n}(s,a)$ denote the population variances obtained with the first $n$ samples.
 It is immediate to verify that $\epsilon_{p,n}^{sas'} \leq \beta_{p,n}^{sas'}$ and $\epsilon_{r,n}^{sa} \leq \beta_{r,n}^{sa}$ \as (see Eq.~\ref{eq:bernstein_confidence_bound_p} and \ref{eq:bernstein_confidence_bound_r} with $N_k(s,a)$ replaced by $n$).
 Using the empirical Bernstein inequality~\citep[][Thm. 1]{Audibert:2009:ETU:1519541.1519712} we have that for all $n \geq 1$:
 \begin{align}
 \Proba{ | p(s'|s,a) - \wh{p}_n(s'|s,a) | \geq \beta_{p,n}^{sas'} } &\leq \Proba{| p(s'|s,a) - \wh{p}_n(s'|s,a) | \geq \epsilon_{p,n}^{sas'} } \leq \frac{\delta}{10 n^2S^2A}\label{eq:high_proba_p}\\
 \Proba{ | r(s,a) - \wh{r}_n(s,a) | \geq \beta_{r,n}^{sa} } &\leq \Proba{ |r(s,a) - \wh{r}_n(s,a) | \geq \epsilon_{r,n}^{sa} } \leq \frac{\delta}{10 n^2 SA} \label{eq:high_proba_r}
 \end{align}
 Note that when $n = 0$ (\ie when there hasn't been any observation of $(s,a)$), $\epsilon_{p,0}^{sas'} \geq1$ and $\epsilon_{r,0}^{sa} \geq \rmaxbound$ so $\Proba{| p(s'|s,a) - \wh{p}_0(s'|s,a) | \geq \epsilon_{p,0}^{sas'} } =  \Proba{ |r(s,a) - \wh{r}_0(s,a) | \geq \epsilon_{r,0}^{sa} } =0$ by definition. 
Since in addition (also by definition) \[ B_p^n(s,a,s') \subseteq  \left[\wh{p}_n(s'|s,a) - \beta_{p,n}^{sas'},\wh{p}_n(s'|s,a) + \beta_{p,n}^{sas'}\right] \text{ (see Eq.~\ref{eq:confidence_interval_p})}\] and \[ B_r^n(s,a)\subseteq \left[\wh{r}_n(s,a) - \beta_{r,n}^{sa}, \wh{r}_k(s,a) + \beta_{r,n}^{sa}\right]    \text{ (see Eq.~\ref{eq:confidence_interval_r})}\] we conclude that for all $n \geq 1$
\begin{align*}
 \Proba{ p(s'|s,a) \notin B_p^n(s,a,s') } \leq \frac{\delta}{10 n^2S^2A}~
 \text{ and }~ \Proba{ r(s,a) \notin B_r^n(s,a) } \leq \frac{\delta}{10 n^2 S A}                                                                                                                                                                                      
\end{align*}
and these probabilities are equal to $0$ if $n=0$.
Plugging these inequalities into Eq.~\eqref{eq:inclusion_decomposition_proba} we obtain:
\begin{align*}
 \Proba{\exists T \geq 1, \exists k \geq 1 \text{ \st } M \not\in {\mathcal{M}}_k} \leq \sum_{s,a} \left(0 + \sum_{n=1}^{+\infty} \left( \frac{\delta}{10 n^2 SA} + \sum_{s'}\frac{\delta}{10 n^2 S^2 A} \right) \right) = \frac{2 \pi^2 \delta}{60} \leq \frac{\delta}{3}
\end{align*}
which concludes the proof.
\end{proof}
\end{document}